\documentclass[12pt]{amsart}
\usepackage[margin=1.2in]{geometry}

\usepackage{hyperref}
\usepackage{cleveref}
\date{\vspace{-5ex}}
\usepackage{amsthm, amsfonts, mathrsfs}
\usepackage{ dsfont }
\usepackage{amssymb}
\usepackage{enumerate}
\usepackage{graphicx}
\usepackage{float}
\usepackage{bbm}
\usepackage{amsmath}
\usepackage{comment}
\usepackage{hyperref}
\usepackage{listings}
\usepackage{color}
\usepackage[dvipsnames]{xcolor}
\usepackage[lined,boxed,commentsnumbered]{algorithm2e}
\usepackage[colorinlistoftodos]{todonotes}
\graphicspath{ {Images/} } 

\allowdisplaybreaks

\hypersetup{
    colorlinks,
    citecolor=blue,
    filecolor=blue,
    linkcolor=blue,
    urlcolor=blue
}

\newtheorem{theorem}{Theorem}[section]
\newtheorem{proposition}[theorem]{Proposition}
\newtheorem{lemma}[theorem]{Lemma}

\newtheorem{corollary}[theorem]{Corollary}
\newtheorem{definition}[theorem]{Definition}
\newtheorem{remark}[theorem]{Remark}

\newtheorem{example}{Example}

\newcommand{\FS}{\mathcal {F}}

\newcommand{\Diff}{\mathcal B}
\newcommand{\id}{ {Id}}

\newcommand{\GM}{ \mathcal S}
\newcommand {\Tr} {\mathcal T}
\newcommand {\Aff}{\mathcal A}
\newcommand{\Monge}{\mathbb M}
\newcommand{\TMP} {\mathbb T}
\newcommand{\MS}{\mathcal W}

\newcommand{\R}{\mathbb{R}}
\newcommand{\Rp}{\R_{+}}

\newcommand{\cH}{\mathcal H}
\newcommand{\Ha}{\mathcal H_a}

\newcommand{\mP}{\mathcal P}
\newcommand{\grad}[1]{\triangledown #1}

\newcommand{\dprime}{\prime\prime}
\newcommand{\fp}[1]{f^{\prime}_{#1}}

\newcommand{\gp}[1]{g^{\prime}_{#1}}

\newcommand{\fpi}{(f^{\prime})^{-1}}
\newcommand{\gpi}{(g^{\prime})^{-1}}

\newcommand{\supp}{{\rm supp\,}}

\newcommand {\bpf} {\begin {proof}}
\newcommand {\epf} {\end {proof}}

\title{Partitioning signal classes using transport transforms for data analysis and machine learning}

\begin{document}
	\date{}

	\author{Akram Aldroubi, Shiying Li, Gustavo K. Rohde
		}

\keywords{convexity, transport-transforms, convex groups, data analysis, classification, machine learning}
\subjclass [2010] { }

\maketitle

\begin{abstract}
A relatively new set of transport-based transforms (CDT, R-CDT, LOT) have shown their strength and great potential in various image and data processing tasks such as parametric signal estimation, classification, cancer detection among many others. It is hence worthwhile to elucidate some of the mathematical properties that explain the successes of these  transforms when they are used as tools in data analysis, signal processing or data classification. In particular, we give conditions under which classes of signals that are created by algebraic generative models are transformed into convex sets by the transport transforms.  Such convexification of the classes simplify the classification and other data analysis and processing problems when viewed in the transform domain.   More specifically, we study the extent and limitation of the convexification ability of these transforms under  an algebraic generative modeling framework. We hope that this paper will serve as an introduction to these transforms and will encourage mathematicians and other researchers to further explore the theoretical underpinnings and algorithmic tools that will help understand the successes of these transforms and lay the groundwork for further successful applications. 
\end{abstract} 

\section{Introduction}

Recently a set of transforms that have close relationship to the mathematics of optimal transport \cite{monge1781,kantorovich1942translation,brenier1991,villani2003topics} have been introduced for representing signals,  images and data. 
In this framework, a signal $p$ is a normalized non-negative function, while its transform $\widehat p$ is a function which is consistent with an optimal transport map. Succinctly, the transform $\widehat p$ of a function $p$ is the transport map that morphs (transports)  a chosen reference function $r$ to $p$. This general idea has led to the development of several types of non-linear transforms for signal and image data, including the Cumulative Distribution Transform (CDT) \cite{park2017}, the Radon CDT \cite{kolouri2016c}, and the Linear Optimal Transport (LOT) \cite{wang2013,kolouri2016b} transforms. 
Numerous problems in signal and image analysis, and in machine learning have been successfully and efficiently solved in
the transport transform domains \cite{kolouri2017optimal}.  
At their core, the aforementioned transforms capture the movement (transport) of the signal amplitude (or pixel intensity) along the independent variable (usually time for {one-dimensional} signals or space for images) in a mathematically rigorous manner.
 This transport encoding of signal or image intensity has enabled numerous interesting applications, many of which were not possible with other, more standard, techniques. The transport-based transforms have been successfully applied to many  data analysis problems including parametric signal estimation \cite{Rubaiyat20} (Rubaiyat et al., {2020}), signal and image classification \cite{park2017,kolouri2016c,kolouri2016b} (Park et al., 2017; Kolouri et al., 2016, Kolouri et al., 2016), modeling turbulence  \cite{emerson2020turbulence}(Emerson et al., 2020), cancer detection \cite{ozolek2014, tosun2015detection} (Ozolek et al., 2014; Tosun et al., 2015), vehicle-type recognition \cite{Guan2019} (Guan et al., 2019), and others \cite{kundu2018discovery}(Kundu et al., 2018).
 Although not based on the transport transforms,  transport maps have also been used in many other successful applications of data analysis such as \cite{haker2004, rubner2000, ni2009local, shen2018wasserstein,arjovsky2017wasserstein,schiebinger2019optimal}.

In the transport transform framework, the set of signals  $\mP_d$ (see \eqref{SigSpaces} below) consists of non-negative functions. Except for the zero signal, they are normalized to have $L_1$-norm equal to 1. For  signal analysis, processing or classification problems, the signals are modeled  by deformations of template signals. For example, a smiling face-image $p_h$ of an individual can be thought of (approximately) as a deformation of a neutral  face-image $p$ of the same individual  via the diffeomorphism $h$, see Figure \ref {fig:intro_face_demo}. A class of signals $\GM_{p,\cH}$ can be generated by applying a set of diffeomorphism $\cH$ to a signal $p$. The process of generating such class from $\cH$ will be called an algebraic generative model. This is different from the statistical generative model in machine learning which assumes  that  the data is generated from an underlying conditional probability distribution, and the aim is to find this distribution.   In contrast, the algebraic generative model creates signal classes from a template and a set of diffeomorphisms that produce the class by morphing the template (no probabilistic assumptions about the classes are used). Under this algebraic generative modeling assumption, the transport transforms can be particularly useful for applications. 

Data analysis, and machine learning methods when applied  in the transport transform domain work best if the data has formed via a transport type phenomena. 
For example, take the task of modeling the difference between images of  smiling faces versus faces with neutral expression as in Figure \ref{fig:intro_face_demo}. Building a mathematical model to automatically recognize the difference between these two classes is still a challenging problem, and  learning a classifier in raw image domain is a difficult task. However, a smiling face is formed by muscle mass movement (transport) from a neutral face.  Thus, in transport transform domain (LOT), this problem becomes simple as shown in Figure \ref{fig:intro_face_demo}. The  two underlying classes (smiling vs neutral) seem to cluster into two disjoint convex sets. Thus classification can be easily achieved using the Hyperplane Separation Theorem (or the Hahn-Banach Separation Theorem). In addition,  because the transform is invertible, any point in transform space can be inverted (see caveats below), and thus any point along the estimated classifier line (i.e., the line orthogonal to the  separating hyperplane in the transform domain) can be visualized in image space. As can be seen from the inverse of the linear classifier, the technique is able to correctly summarize the differences between the two classes by providing an ``average" face that goes from neutral to smiling as one traverses along the  classifier line. This idea of utilizing transport-based transforms to perform morphometry operations was first introduced in \cite{wang2013,basu2014} where the goal was to combine learning techniques with transport-based representation techniques to decode important trends in a given dataset.
 
 As in the previous example, it has been demonstrated \cite{park2018, kolouri2016c} that one of the main advantages of solving estimation and detection (i.e. classification) problems in transport-transform (CDT, R-CDT, LOT) domain relates to rendering signal and image classes  convex. The idea is outlined in Figure~\ref{fig:linear_separability}. Two classes in signal space $\mathbb{P}$ and $\mathbb{Q}$ are displayed. They consist of signals containing ``one bump" and ``two bumps", respectively, observed under random translation along the independent variable. The right panel shows the same classes expressed in transport transform domain. It is clear that transforming the set of 1D signals has rendered the problem easier to solve, given the ensuing manifold (the geometric structure encompassing the transformed data) becomes linear in this example.
 {Furthermore, it is possible to  use  the convexity property of the generative model in the transform domain to better model a  data class when the set of training data has few samples (e.g., with limited training data). This is an important property for data classification problems and we will expand on this concept in the concluding remarks of  the summary and open questions (Section \ref{Sec: conclusion}).}
  Recently, Shifat-E-Rabbi et al. \cite{shifaterabbi2020radon}, have made use of such properties to propose a transport transform nearest subspace method for image classification. Empirical tests show the method, which is simple to compute and does not require iterative tuning of hyperparameters, can rival that of popular neural network type classifiers in certain problems, for a small fraction of the computational cost and with fewer training samples.

\begin{figure}[!hbt]
    \centering
    \includegraphics[width=13cm]{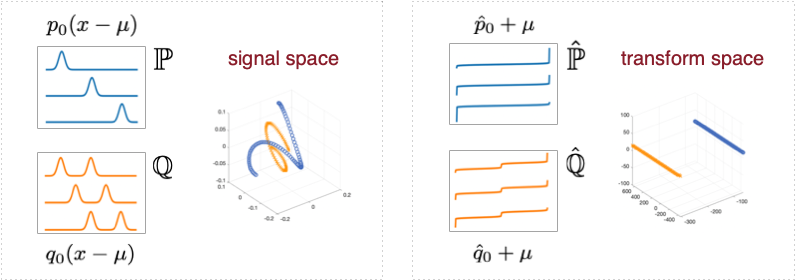}
    \caption{An example of a signal classification problem. In this example, class 1 consists of random translations of a ``one bump'' signal and class 2 consists of random translations of a ``two bump" signal.}
    \label{fig:linear_separability}
\end{figure}
 

\begin{figure}[!hbt]
    \centering
    \includegraphics[width=13cm]{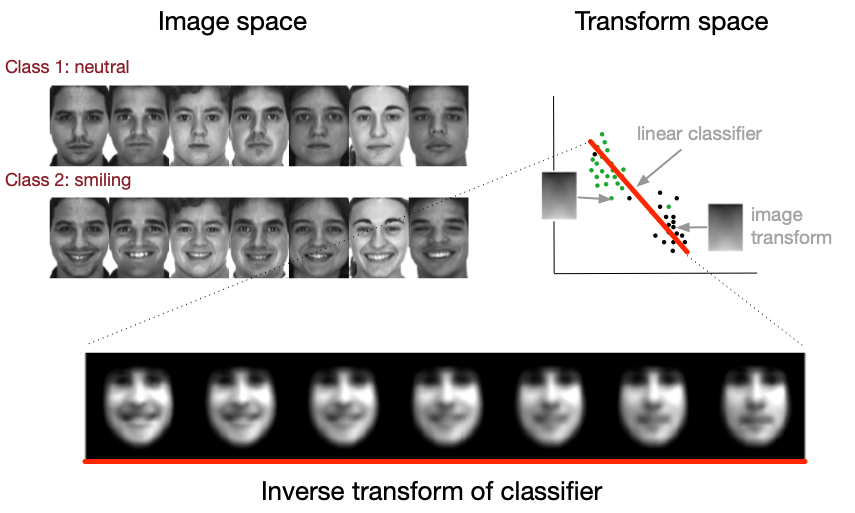}
    \caption{ Automated modeling differences between smiling and neutral facial expressions using transport transforms. Top left: sample images of both neutral and smiling classes. Top right: depiction of data and a linear classifier (line in red) in transform domain.  Bottom: given that the transforms are inverible, the inverse of the classifier can be obtained. As shown by the inverse transform of the classifier line, the classifier consists of an ``average'' image, while different directions along the discriminant line (in red) represent how wide a person smiles. The technique is able to correctly summarize the differences between the two classes by providing an ``average" face that goes from neutral to smiling as one traverses along the  classifier line. }
    \label{fig:intro_face_demo}
\end{figure}

\subsection{An overview of transport transforms}

{Data analysis and processing algorithms often assume  mathematical models of how the data {have} been generated. For example, data compression algorithms work well using thresholding in wavelet transform domain if the data is assumed to have  large smooth regions. Similarly,  analysis,  synthesis, classification and other algorithms tend to work best in the transport-based transforms domains when the data to be processed is generated by physical processes that include transport. This mathematical model on the data sets is the so-called \emph{Algebraic Generative Model, and it is described below}.}

    \subsubsection{Algebraic generative model \texorpdfstring{$\GM_{p,\cH}:= \{ p_h=|\det J_h|\cdot p\circ h  \mid  h \in \cH\}$}{genmodel} } {In this model a signal $p_h$ is generated by applying a differentiable one-to-one spatial transformation $h\in \cH$, to a fixed (but often unknown) template $p$ via $p_h=|\det J_h|\cdot p\circ h$, where $\cH$ is a set of diffeomorphisms and $J_h$ is the Jacobian matrix of $h$.  For example, when $\cH$  is a set of translations, the setup is useful to model time delay estimation and tracking problems \cite{Rubaiyat20}. When $\cH$ includes nonlinear deformations, it can be used to model mass (e.g. molecule) concentrations inside cells \cite{basu2014}, tissues in human brains \cite{kundu2018discovery}, photon distributions in turbulent media \cite{park2018multiplexing}, and others \cite{kolouri2017optimal,Rubaiyat20}. 
    The complete mathematical specification of the algebraic generative model is available in Section \ref{subsec: alg model}.}
    \subsubsection{Transport transforms}  {Given a signal $s$, its transform is defined by the optimal transport map (cf. Section \ref{sec: Monge}) between the signal $s$ and a chosen reference function $r$. The transform defined this way yields a number of interesting properties that facilitate the solution to many data analysis problems when the data is consistent with  algebraic generative models as above. In particular, if the data at hand arises from the generative model outlined above (described in more detail in Section \ref{subsec: alg model}), the set of signals will form a convex set in transform space. This has useful implications for estimation and classification problems as described below. }

\subsection{Contributions}
  The ability of the emerging transport-based transforms to facilitate the solution of  certain estimation, classification, data processing and analysis problems is related to the extent to which they are able to represent the data as convex sets. As such, this manuscript is devoted to clarifying the ability of different transport-based transforms to render the classes produced by algebraic generative models into convex sets. More precisely, we specify conditions which will render signal or image classes convex in transform domain. 
  \subsubsection{Convexity condition for one-dimensional generative models}
  {We characterize the exact condition under which a set of signals produced by the algebraic generative models (one-dimensional) become convex in the transform (CDT) domain. Specifically, the set of transformed signals, denoted by $\widehat{\GM}_{p,\cH}$, is convex for all $p$ if and only if the set $\cH^{-1}$ of transformation diffeomorphisms is convex (cf. Theorem \ref{ConvexGroup1dSet} and Corollary \ref{ConvexGroup1d}). Moreover, when $\cH$ is a convex group, the transform space $\widehat \mP_1$ (cf. Table \ref{TblSymbl}) can be partitioned into convex equivalent classes. We give various examples of convex groups $\cH$ and show that there are infinitely many of them (cf. Section \ref{subsec: 1dConvGrp}). }
  
 \subsubsection{Convexity condition and limitation for multi-dimensional generative models}
{For dimension $d\ge 2$, the situation is more complicated than for the case $d=1$. For this case, we only provide a sufficient condition under which the algebraic generative models ${\GM}_{p,\cH}$ become convex in the transform domain. Specifically, if  the  set $\cH^{-1}$  of transformation diffeomorphisms is a convex subset satisfying $\widehat p_h = h^{-1}\circ \widehat p$ for all $h\in \cH$ and for all $p\in \mP_d$ (cf. Equation \eqref{eq:d-comm} and Table \ref{TblSymbl}), then $\widehat{\GM}_{p,\cH}$ is convex for any $p$ (cf. Theorem \ref{ConvexGroupmdSet}).  Equation \eqref{eq:d-comm}  will be referred as the composition property for CDT \cite{park2017}. 
When $d\geq 2$, if we require   that the composition condition on $\cH$  holds for all $p$, then we show that $\cH$ must be a subset of translations and isotropic scaling diffeomorphisms (cf. Theorem \ref{thm: nD_group}).  However, by relaxing the composition requirement to hold on certain subsets of signals,   the {set} of diffeomorphisms $\cH$ that guarantee the aforementioned convexity result can be expanded. In particular, we give a relaxation in dimension two (cf. Section \ref{sect: relax}) where a set  $\cH$, larger than the set of all translations and isotropic scaling diffeomorphisms (cf. Remark \ref{rmk: propchr}), can guarantee the convexity results in transform domain if the data of interest conform to the generative model in a more restrictive sense (cf. Theorem \ref{thmn2hatformula}).} 
 
 \subsubsection{Practical implications}
  {The convexity results that we present in this manuscript have two immediate practical implications for image and signal processing.}

  {A first implication is the simplification of classification problems in the sense that there exists a linear classifier that can perfectly separate disjoint data arising from the aforementioned algebraic generative model in transform space. For example, the sets generated by simple models that involve only translations (see e.g., Figure \ref{fig:linear_separability}) can have a complex geometry in the signal domain and generally are not linearly separable. Compared to the unknown (usually quite complicated) geometry in the signal domain, the convexity in transform domain guarantees the existence of a linear classifier that will perfectly separate the data classes. While this  manuscript does not prescribe a particular method to separate { convex data sets} in transform domain, numerous machine learning methods for that purpose exist \cite{hastie2009elements}{;} specifically linear support vector machines \cite{cortes1995}, Fisher discriminant analysis \cite{fisher1936,Belhumeur97}, linear logistic regression \cite{mccullagh1989generalized} are popular algorithms, which perform differently depending on how the data is statistically distributed over the same geometry, can be used.}

  {As a second implication, the property that the set of a transformed signal class is convex allows us to solve many interesting estimation problems. Specifically, this property provides a necessary condition for designing linear least-squares techniques in the transport transform domain  which, otherwise would necessitate nonlinear and nonconvex optimization and thus be difficult to solve.   Rubaiyat et al.\cite{Rubaiyat20}, for example, assume that the measured signal is deformed from the target signal via transformations that lie in a finite linear subspace (in particular, space of polynomials of a certain degree) and obtained fast and accurate estimation such as time delay and quadratic dispersion parameters in various applications.}
  

\subsection{Paper organization}
The rest of the paper is organised as follows. In Section \ref{Sec: prelim}, the transport-based transforms, their connections to the optimal transport theory and an associated generative model are introduced. In Section \ref{Sec: oneD}, convexification results (Proposition \ref {Part}, Theorem \ref {ConvexGroup1dSet} and corollaries) and examples of the CDT for one-dimensional  generative models are presented. In Section \ref{Sec: multi-D}, we present the limitations on the generative model with respect to convexification by the LOT in dimension $d\geq 2$ (Theorem \ref {thm: nD_group}) and a possible relaxation to mitigate the limitations in dimension two (Theorem \ref {thmn2hatformula}). A more detailed summary of results and a discussion of open questions are given in Section \ref{Sec: conclusion}.

\section{Preliminaries, Notation and Model Assumptions}\label{Sec: prelim}
In this section, we define the various transforms, their domains, ranges, and their connections to optimal transport theory,  and introduce a signal model of interest.

\subsubsection{Signal and Transform spaces}
The signal spaces we consider consist of non-negative Lebesgue measurable functions that are compactly supported and normalized, i.e., a non-zero signal $p$ is first normalized to have its $L_1$-norm $\|p\|_1=1$. The space of all signals is formally  described by  

\begin{equation} \label{SigSpaces}
	\mP_d: = \{p: \R^d\rightarrow \Rp \mid \supp (p)=\Omega_p~\textrm{is compact}, ~  \int_{\R^d} p(x)dx=1\}.
\end{equation}
Note that the signals in the above set have finite moments since they are compactly supported. Since signals in  $\mP_d$ have $L_1$-norm equal to 1, it is sometimes useful to think of them as probability density functions.
The class of transforms $\widehat \mP_d$ belongs to the set of functions $\FS_d$ that are solutions to the Monge Transport Problem discussed in Subsection \ref{sec: Monge} below: 
\begin{equation}\label{multiDsig}
	\FS_d: =\{f: \R^d\rightarrow \R^d \mid f = \grad \phi ~ \textrm{for some convex}  ~ \phi : \R^d \rightarrow \R \}.
\end{equation}
In particular, $\FS_1$ consists of all a.e. non-decreasing functions from $\R$ to $\R$.

\subsection{Transforms and the Monge problem}\label{sec: Monge}
The general theory of optimal transport is a deep and well-developed area that started with a transport problem due to Monge \cite{monge1781}. The theory and solutions to this problem took many detours and gave rise to a general theory of optimal transports which was spearheaded and developed by Kantorovich, Brenier, Villani and many others in the last two hundreds years \cite{kantorovich1942translation,brenier1991,villani2003topics,santambrogio2015optimal}. {For a quick introduction to the key concepts in optimal transport theory, see \cite{thorpenotes} by Thorpe.}  For our purpose we use one of the simplest results of the theory to  introduce the needed transforms.

Given a fixed reference function $r \in \mP_d$, the transform $\widehat p$ of $p \in \mP_d $ is the unique solution to the Monge optimal transport problem:
\begin{equation} \label{Monge}
 \text {minimize  } \Monge (T)=\int_{\R^d}\big|x-T(x)|^2r(x)dx   
\end{equation}
over all  $T$ that makes the push-forward relation below hold:
\begin{equation}
\label{pushforward} 
\int_B p(y)dy=\int_{T^{-1}(B)} r(x)dx,
\end{equation}
holds for every measurable set $B.$  In measure theory, the relation \eqref{pushforward} above is written more compactly as
\begin{equation} \label{pushforwardmeas} 
\mu_p=T_\#\mu_r,   
\end{equation}
where $d\mu_p=pdx,\;  d\mu_r=rdx$. 
When $T$ is a $C^1$ diffeomorphism, the constraint above becomes
\begin{equation} \label{MassCons}
r(x)= |\det \big(J_T\big)|p\big(T(x)\big),
\end{equation} 
where $J_T$ is the Jacobian matrix of $T$.

The existence and uniqueness of solutions to the Monge problem under the assumptions of the signal spaces $\mP_d$ in this paper is a special case of the well-known Brenier's Theorem  described below. The function $T$  in Equation \eqref{pushforwardmeas} can be interpreted as a mass preserving map between the reference signal $r$ (a reference mass distribution) and the given signal $p$ (a mass distribution described by $p$). With this point of view, Monge's Problem can then be interpreted as finding the optimal transport map that will transform a mass distribution to another distribution of equal mass. 

Let $r\in \mP_d$ and $\TMP_r:\mP_d\to \FS_d$ be the operator that maps an element $p \in \mP_d$ to its optimal transport map described above. Under the assumptions on $\mP_d$, we define the transport transform of a function $p$ with respect to the reference $r$ as 
\begin{equation} \label{TrprtTrfrm}
    \widehat p=\TMP_r(p).
\end{equation}

\begin{definition} [CDT, LOT, and R-CDT] ${}$\label{CDTLOT}
\begin{enumerate}
    \item when $d=1$ in \eqref{TrprtTrfrm}, $\TMP_r$ is called the CDT. 
    \item when $d\ge 2$, $\TMP_r$ is called the Linear Optimal Transport (LOT) transforms (although the transform itself is nonlinear).
    \item The R-CDT consists of the composition of the Radon Transform and the CDT transform.
\end{enumerate}
\end{definition}

\begin{remark}
In one dimension, as mentioned above, the optimal transport maps are non-decreasing functions defined on $\R$. In particular, given $p\in \mP_1 $, its CDT transform $\widehat p$ with respect to a reference $r$ can be equivalently defined through the following relation\footnote{Note that $\widehat p(x)$ defined in \eqref{EqivCDTGen} is finite $\mu_r$ a.e., but could be $+\infty$ for some $x$.}
\begin{equation} \label{EqivCDTGen}
    \widehat p(x)= \sup\Big\{t\; \mid \int_{-\infty}^{t} p(\xi)d\xi \le \int_{-\infty}^{x}r(\xi)d\xi\Big\}. 
\end{equation}
If $\textrm{supp}(p)$ is an interval and $p$ is continuous on $
\supp(p)$, then \eqref{EqivCDTGen} simplifies to
\begin{equation} \label{EqivCDT}
    \int_{-\infty}^{\widehat p(x)} p(\xi)d\xi = \int_{-\infty}^{x}r(\xi)d\xi, \quad \forall x\in \supp(r).
\end{equation}
The equivalent formulation \eqref {EqivCDT} can be derived by integrating Equation \eqref{MassCons}, replacing $T$ by $\widehat p$, and making use of the property that $\widehat p$ is increasing a.e. $\mu_r$.
\end{remark}

 Under the assumption on the signals spaces,  all the transforms above are non-linear and are injective a.e. $\mu_r$.  See Figure \ref{fig:CDT} for examples of CDTs with the reference $r$ being the characteristic function on $[0,1].$  

\begin{figure}[!hbt]
    \centering
    \includegraphics[width=13cm]{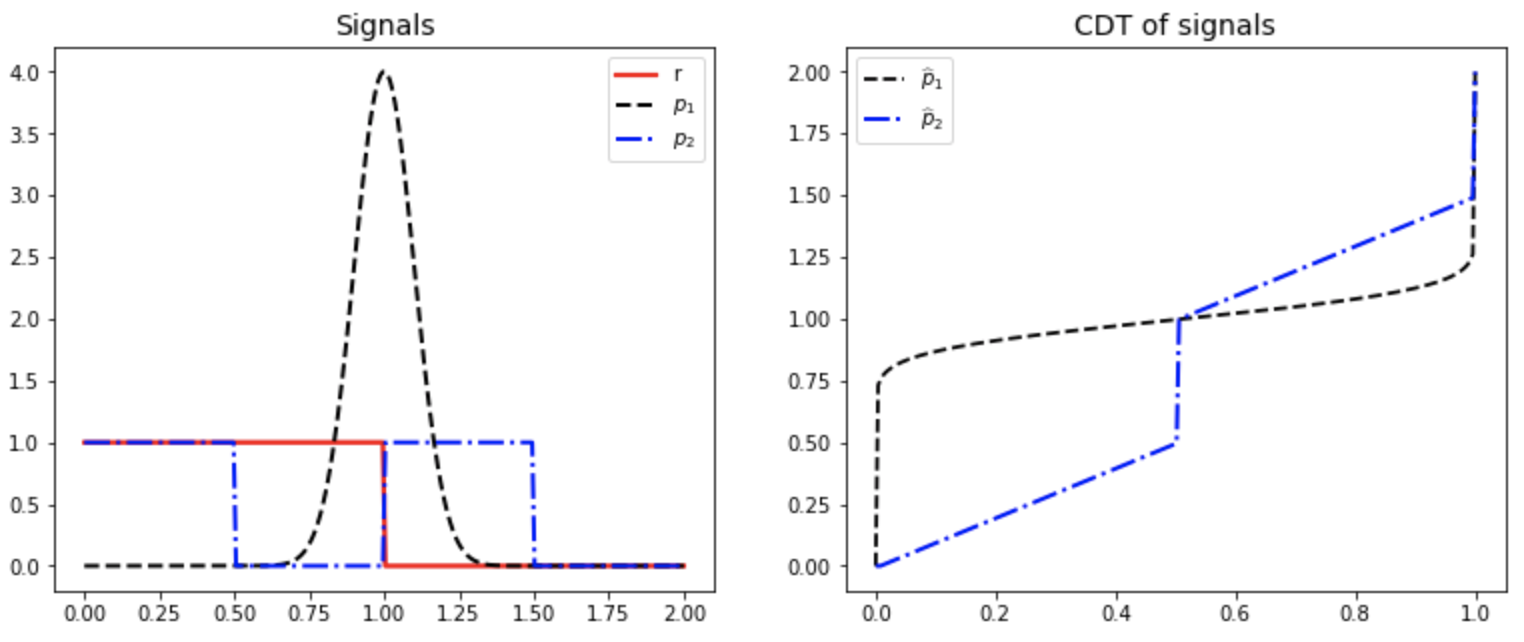}
    \caption{CDT examples with $r = \chi_{[0,1]}$. Left panel: the reference function $r$ (red), a Gaussian like signal (black), and the sum of two characteristic functions (blue). Right panel: the transform of the Gaussian like function (black) and the transform of the sum of two characteristic functions (blue).}
    \label{fig:CDT}
\end{figure}
A summary of the various spaces and symbols is provided in Table \ref{TblSymbl}. 
\begin{table*}[!hbt] 
\centering
\normalsize
\caption{Notation}
\label {TblSymbl}
\begin{tabular}{ll}
\hline
Symbols                & Description    \\ \hline
$p$,~$q$,~$r$ & normalized functions on $\R^d$\\
$\Omega_p$ & Support of a function $p$ \\
$\mP_d$ &  the set of compactly supported {non-negative} normalized functions  \\
$\FS_d$& the set of optimal transport maps from $\R^d$ to $\R^d$\\
$\FS_d^G$& the set of $C^1$ diffeomorphisms in $\FS_d$ \\
$\cH$ & a subgroup of  $\FS_d^G$ \\
$\cH_a$ & the group of translation and isotropic scaling diffeomorphisms on $\R^d$ \\ 
$\GM_{p,\cH}$ & a signal class generated by template $p$ under the diffeorphisms in $\cH$\\
$\widehat p$& the CDT/LOT transform of $p$ with respect to a fixed reference $r$\\
$\widehat \GM_{p,\cH}$ & the transformed signal class $\{\widehat p_h \mid  p_h \in \GM_{p,\cH}\}$\\
$\TMP_r$ & the CDT/LOT transform operator:  $\TMP_r(p)= \widehat p$\\
{$\widehat \mP_d$} & {transform space $\{\widehat s \mid  s\in \mP_d\}$ }\\
\hline
\end{tabular}
\end{table*}
\subsubsection{Connection with the Wasserstein distance}
{Given two probability measures $\mu$ and $\nu$,  the Wasserstein-$2$ distance between them is defined as 
\[W_2(\mu,\nu):= \big(\inf\limits_{\pi \in \Pi(\mu,\nu)}\int_{\R^d\times \R^d} |x-y|^2 d\pi(x,y) \big)^{\frac 1 2},
\]
where $\Pi(\mu,\nu)$ is the set of measures on $\R^d\times \R^d$ with $\mu$ and $\nu$ as marginals \cite{villani2003topics}. 
In the one-dimensional case, one can show that the CDT transform defines an embedding from $\mP_1$ with the $W_2$-metric to the transformed space $\widehat \mP_1$ \cite{park2017}. In particular, $W_2(\mu_p,\mu_q) = ||(\widehat p-\widehat q)\sqrt{r}||_{L^2}$ for any $p,q\in \mP_1$. However, when $d\ge 2$, this embedding property does not hold in general. The Euclidean-type distance $||(\widehat p-\widehat q)\sqrt{r}||_{L^2}$ is referred as the linearized optimal transport (LOT) between $p$ and $q$  when $d\ge 2$ and has been shown useful in image pattern recognition, discrimination and visualization problems \cite{wang2013,kolouri2016b}. }

\subsection{Optimal transport maps}\label{subsec: op maps}
We start with a special case of Brenier's Theorem (see e.g., \cite{brenier1991,santambrogio2015optimal,villani2003topics} and the references therein) which we need for this investigation. In particular, using Theorem 2.12 in \cite{villani2003topics} and Theorem 1.48 in \cite{santambrogio2015optimal} we have
\begin{theorem}[Brenier's Theorem]\label{thmBrenier}
Let $p,r \in\mP_d$. Then there exists a unique solution $T \in \FS_d$ (up to sets of $\mu_r$-measure zero) to the \textit{Monge transport problem} associated with $r,p$ and the cost function $c(x,y)= |x-y|^2$. Conversely, let  $r \in\mP_d$ and $T \in \FS_d$ {such that $\int_{\R^d} |T(x)|^2r(x)dx<\infty$}. Then $T$ is optimal for  the Monge Problem above for the function $p\in\mP_d$ satisfying $\mu_p=T_\#\mu_r$.
\end{theorem}
\begin{remark}
Brenier's Theorem is much more general than the version described above. Specifically, Brenier's general Theorem describes the situation where $r,p$ in the Monge Problem are replaced by two probability measures $\mu,\nu$ that do not necessarily have associated density functions and where  the term  $|x-y|^2$ is replaced by a more general term $c(x,y)$. 
\end{remark}

\subsection{Algebraic generative models}\label{subsec: alg model}
{In designing algorithms for data analysis, processing or classification, a mathematical model of how the data {have} been generated is often assumed.  For example, algorithms that work well on the wavelet transforms of images often assume the underlying class of images is well modeled by functions that have Fourier transforms that are well concentrated in some regions of frequency domain (e.g., near the origin). In a similar way, processing algorithms work best {with} the transport-based transforms when the data to be processed {are} generated by physical processes that include transport. Examples where algebraic generative models combined  with transport-based transforms have been used effectively in applications including classification of cancerous vs normal cells \cite{basu2014}, pattern analysis of tissues in human brains \cite{kundu2018discovery}, decoding optical communications in turbulent media \cite{park2018multiplexing}, and others \cite{kolouri2017optimal,Rubaiyat20}}.

In contrast with the generative model definition typically used in machine learning, the algebraic models assume that a class of functions $\GM \subset \mP_d$ is  generated from a  function $p\in \mP_d$ by a mass transport phenomena. Although there are infinitely many ways to transport $p$ to form $q \in \GM$,  absent other information related to the generative process, the least action principle from physics often provides one with plausible solutions.  In these circumstances a set $\cH \subset \Diff$ of optimal transport maps can be used as a generative model, where $\Diff$ is set of $C^1$-diffeomorphisms on $\R^d$.  In summary, the set of optimal transport maps used for our generative modeling is 
\begin{equation}\label{multiDsigmodel}
	\FS_d^G: =\{f\in \Diff \mid f = \grad \phi ~ \textrm{for some convex}  ~ \phi : \R^d \rightarrow \R \}.
\end{equation}
Given a function $p\in \mP_d$ (signal), a class of functions (signals) is generated by the action of a set of diffeomorphism $\cH\subset \FS_d^G$ on $p$ via
\begin{equation} \label {GenClass}
	\GM_{p,\cH}:= \{ p_h=|\det J_h|\cdot p\circ h  \mid  h \in \cH\},
\end{equation}
where $J_h$ denotes the Jacobian matrix of $h$.
It is not hard to check that $\GM_{p,\cH}\subset \mP_d$.
One of the simplest examples of such a generative model is described in Figure \ref {fig:linear_separability}. 
A natural assumption on the set of diffeomorphisms $\cH$ used in the generative model is that it has a group structure (see \cite{park2017}): 
\begin {enumerate}
\item $\cH$ is closed under composition.
\item  $\id\in \cH$.
\item if $h \in \cH$, then $h^{-1}\in \cH$.
\end {enumerate}
In this generative model, the set $\cH$ has a group structure that does not depend on $p\in \mP_d$. However, other  generative models are possible where $\cH$  does not have a group structure, or where  $\cH=\cH(p)$   depends on the initial function $p$. {In practice, there are applications which are more appropriately modelled using an $\cH$ that has a group structure and also situations where a group structure  is  not needed.  General convexity results without assuming $\cH$ to be a group will  also be presented in later sections.}

\section {cdt and generative models   in one dimension}\label{Sec: oneD}
In this section we consider a set of transformations that act on one-dimensional functions (the signals) and produce  classes of functions. For this situation \eqref {GenClass} becomes
\begin{equation}
	\GM_{p,\cH}:= \{ p_h= h^\prime (p\circ h)  \mid  h \in \cH\}.
\end{equation}
Given $p\in \mP_1$, and a diffeomorphism $h\in \FS^G_1$ and a function  $p \in \mP_1$, a new function $p_h \in \mP_1$  is generated via the formula 
\begin{equation} \label{GenMod1deq}
  p_h=h^\prime (p\circ h)  
\end{equation}
which is \eqref{GenClass} for  the one-dimensional case. Let $r\in \mP_1$ be a fixed reference and denote by $\widehat p$ the transform of $p \in \mP_1$ as in \eqref {TrprtTrfrm}.  

Under the definitions above, if we assume that the generative model uses a  subgroup  $\cH$ of  $\FS^G_1$ to generate  $\GM_{p,\cH}$, then  any  function in 	$q\in \GM_{p,\cH}$ can be used as a template to  generate $\GM_{p,\cH}$ by transport diffeomorphism from $\cH$, i.e.,  $\GM_{p,\cH}=\GM_{q,\cH}$ for any $q\in \GM_{p,\cH}$. 

Every subgroup  $\cH \subseteq \FS^G_1$ will generate a partition of $\mP_1$. Thus, in principle, the group structure in the generative model allows us to classify the image (functions) in $\mP_1$, hence can be used in data analysis tasks related to classification. 

\begin {proposition}\label {Part} Every subgroup  $\cH \subseteq \FS^G_1$ partition the set $\mP_1$ via the equivalence relation $\sim_{\cH}$ defined by $p\sim_{\cH}q$ if and only if $p\in \GM_{q,\cH}$.   
\end{proposition}

Often times, if the set of transforms $\widehat{\GM}_{p,\cH}$ is convex, then the solutions of certain estimation problems in transform domain become simple (e.g. linear least squares \cite{Rubaiyat20}). Moreover, this property can also enable easier classification when two disjoint generative classes can be easily separated \cite{shifaterabbi2020radon}. Thus, one of the main goals is to identify conditions under which the set $\widehat{\GM}_{p,\cH}$  of transforms of $\GM_{p,\cH}$ is convex. We have the following theorem.
\begin{theorem} \label{ConvexGroup1dSet}
  Let $\MS\subset \FS^G_1$. Then $\widehat{\GM}_{p,\MS}$ is convex for every $p\in \mP_1$ if and only if $\MS^{-1}:=\{s^{-1}\mid \; s \in \MS \}$ is convex. 
\end{theorem}

\begin{corollary} \label {ConvexGroup1d}
  Let $\cH\subset \FS^G_1$ be a group. Then $\widehat{\GM}_{p,\cH}$ is convex for every $p\in \mP_1$ if and only if $\cH$ is convex. 
\end{corollary}

Obviously $\FS^G_1$ is itself a convex group. Thus, it partitions $\widehat \mP_1$ into equivalent classes that are convex. Two equivalent classes $\widehat S_{p,\FS^G_1}, \;\widehat S_{q,\FS^G_1}$ are distinct if the $\supp (p)$ and $\supp (q)$ are topologically distinct (i.e., non-homeomorphic). Subgroups of $\FS^G_1$ will further partition each  class $\widehat S_{p,\FS^G_1}$ into sub-classes and so on. Even when a generative model $\cH \subset \FS^G_1 $ is a subset but not a subgroup of $\FS^G_1$, it is still true that $\text{conv}(\widehat S_{p,\cH})\cap \text{conv}(\;\widehat S_{q,\cH})=\emptyset$ if $\supp (p)$ and $\supp (q)$ are topologically distinct. An illustration is presented in Figure \ref{fig:LDA_1d}, 
where Linear Discriminant Analysis (LDA) \cite{fisher1936} is applied to two generative classes and their corresponding CDTs in the transform domain. The two signal classes are generated with templates (one-bump characteristic function $p_1$ and two-bump characteristic function $p_2$) shown in Figure \ref{fig:gen_classesv2}  and a set $\cH$ of 500 randomly generated fifth degree polynomials with certain constraints on the coefficients\footnote{The coefficients of the polynomials are chosen so that the supports of $h^{\prime}(p_1\circ h)$ and $h^{\prime}(p_2\circ h)$ are inside the interval $[0,1]$ for all $h\in \cH$, see Figure \ref{fig:gen_classesv2} for some randomly chosen samples from each class.}. It can be seen from Figure \ref{fig:LDA_1d} that the transformed signal classes $\widehat S_{p_1,\cH} $ and $\widehat S_{p_2,\cH}$ (right) are much better separated than the original signal classes $S_{p_1,\cH}$ and $S_{p_2,\cH}$ (left). Moreover, as predicted by our theory above, $\textrm{conv}(\widehat S_{p_1,\cH})$ and $\textrm{conv}(\widehat S_{p_2,\cH})$ are indeed disjoint. This property of the CDT makes it well-adapted for many applications of data analysis, processing and classification. Thus, one of our goals is to understand the structure of the convex subgroups of $\FS^G_1$.
\begin{figure}[!hbt]
    \centering
    \includegraphics[width=14cm]{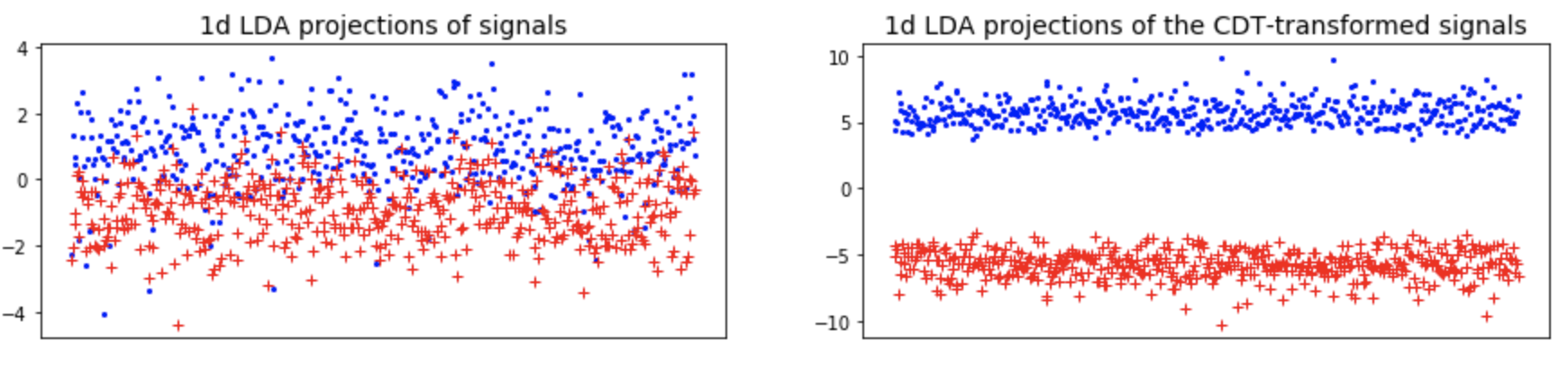}
    \caption{LDA projections of signals in two signal classes and their CDTs. Left panel: LDA projection of 500 signals generated from the top left signal in Figure \ref{fig:gen_classesv2}, and the bottom left signal in Figure \ref{fig:gen_classesv2} respectively. Right panel: LDA projection of the CDT transforms of the 500 signals from each class. Note that the horizontal axes here are dummy axes, which indicate the counting indexes (ranging from $1$ to $500$) of the samples.} 
    \label{fig:LDA_1d}
\end{figure}

\begin{figure}[!hbt]
    \centering
    \includegraphics[width=14cm]{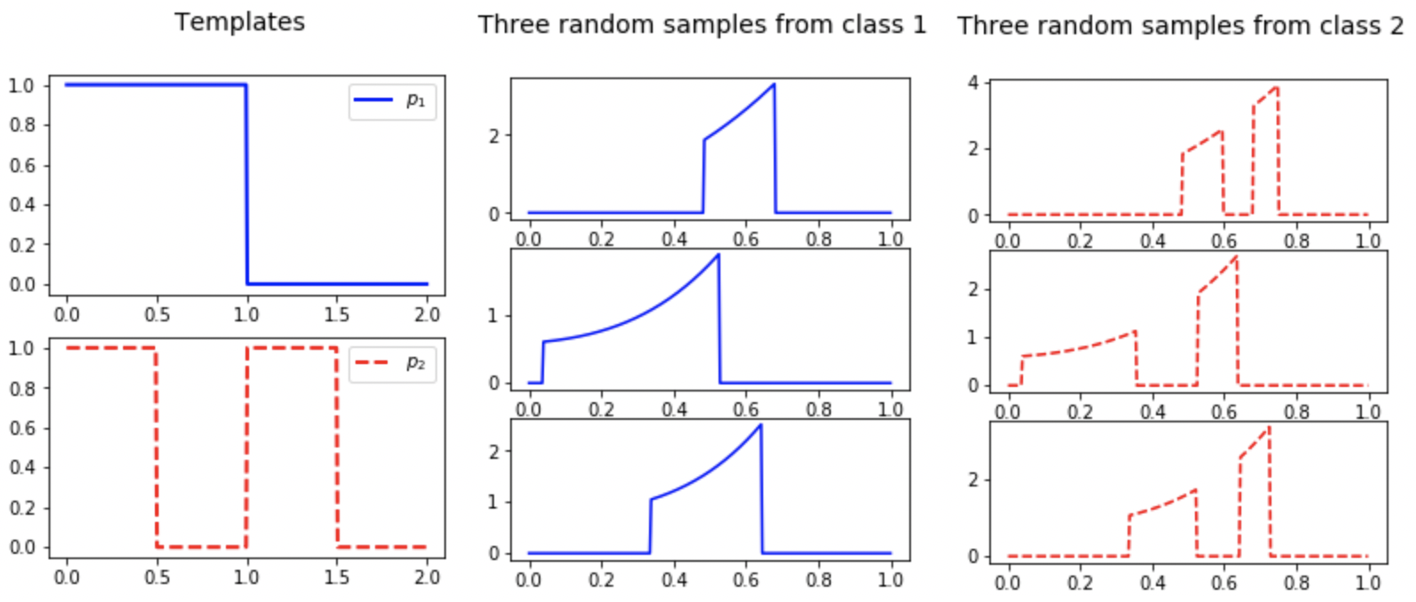}
    \caption{Templates and samples for two generative classes. Right panel: Signal one (blue), signal two (red). Middle panel: Three signals generated by a generating set of diffeomorphisms of top signal in left panel. Right panel: Three signals generated by a generating set of diffeomphisms of bottom signal in left panel. }
    \label{fig:gen_classesv2}
\end{figure}

\subsection {Examples of convex subgroups of \texorpdfstring{$\FS^G_1$}{FSG1}}\label{subsec: 1dConvGrp}
    Note that not every subgroup of $\FS^G_1$ is convex, for example, the group generated by the integer translation diffeomorphisms $f \in \FS^G_1$, i.e., $\{h_{i}\in \FS^G_1 \mid h_{i}(x)= x-i, i\in \mathbb Z\}$. {In addition, there are convex sets of transformations that are relevant for applications but that may not form a group, e.g., the set of polynomials up to degree two. In particular, compositions of quadratic polynomials are not quadratic and a  polynomial of degree two, is not invertible in $\R$. However, for certain applications, requiring the  transformation to be invertible in a restricted domain is enough \cite{Rubaiyat20}. For example, a quadratic polynomial $f(t)$ where $t$ represents time  may arise in the radar motion estimation problem when the car is not moving with a constant velocity, which corresponds to the transport transformation between the source signal and the received signal. In this case, $f(t)$ only needs to be invertible when restricted  to the positive real line. Though the space of polynomials up to degree two is not a convex group, with a fixed source signal,  the model can be used to include all variations of received signals  under time delays, linear and quadratic dispersions of the source signal.  Nevertheless, the convexity of the transformed signal class is guaranteed to be convex so long as  $\cH^{-1}$ is  convex (cf. Theorem \ref{ConvexGroup1dSet}). This property allows one to find fast and practical solutions to  estimation problems when $\cH^{-1}$ is the space  of polynomials in \cite{Rubaiyat20} via a simple linear least squares procedure in transform domain. On the other hand, certain data classes (e.g.,  the MNIST data set) are more appropriately modelled using a set $\cH$ with a group structure. In particular, if any signal in a data class $\GM_{p,\cH}$ can be as good a template as any other one in the class, it makes sense to assume that compositions and inverses of transformations in $\cH$ remain in $\cH$,  indicating that $\cH$ is indeed a group.} Examples of convex subgroups of $\FS^G_1$ are given below.
    
    \begin{example} \label{exsubgrp} ${}$
     \begin {enumerate}
\item $\{\id\}$ is a convex subgroup of $\FS_1^G$. 
\item $\{\alpha \id\mid \alpha >0\}$ is a convex subgroup of $\FS^G_1$ which is also a cone. {As a simple example, this group can be used to model linear dispersion in acoustic or radar signals \cite{Rubaiyat20}.}
\item Let $\Tr_1=\{h_\mu\in \FS^G_1\mid h_\mu(x)=x-\mu, \mu \in \R\}$. The set of all translation functions is a convex subgroup of $\FS^G_1$ but not a cone. {As a simple example, this group can be used to model time delay in a signal class \cite{nichols2019time}.}
\item Let $\Aff_1=\{h_{\alpha,\mu}\in \FS^G_1\mid h_{\alpha,\mu}(x)=\alpha x-\mu, \alpha >0,  \mu \in \R\}$. The set of all increasing affine functions is a convex subgroup of  $\FS^G_1$ and a cone. {This group can be used to model both time delay and linear dispersion in a signal class \cite{Rubaiyat20}.}
\item Let $x_0\in \R$ and consider the set $\FS_{x_0}=\{f\in \FS^G_1\mid f(x_0)=x_0\}$. Then $\FS_{x_0}$ is a convex subgroup of $\FS^G_1$. It is also a cone.
\item Let $\Omega$ define a closed interval in $\mathbb{R}$. Then the set $\FS_\Omega = \{f \in \mathcal{F}_1 \mid f(y) = y \; \forall y \in \Omega \}$ is a convex subgroup of $\FS^G_1$.
\end {enumerate}   
    \end{example}
\begin {remark} 
Let $C_1,C_2 \subset \FS^G_1$ be convex subgroups of $\FS^G_1$. Then it is not difficult to see that if $C_1\cap C_2\neq \emptyset$, then $C_1\cap C_2$ is a convex subgroup of $\FS^G_1$. In fact, let $\{C_\alpha\mid \alpha \in A\}$ be a family of convex subgroups. If $\bigcap\limits_{\alpha\in A} C_\alpha \neq \emptyset$, then $\bigcap\limits_{\alpha\in A} C_\alpha$  is a convex subgroup of $\FS^G_1$. 

\end{remark}
 The following proposition shows that the set of subgroups of $\FS^G_1$ is sufficiently rich.
\begin{proposition} \label {InfSubgrp}
There are uncountably many distinct  convex subgroups of $\FS^G_1$.
\end{proposition}
By examining the proof of the proposition \ref {InfSubgrp}, we get the following corollary.
\begin {corollary}
Let $I$ be an indexing set and let $X=\{x_\alpha \subset \R: \alpha \in I\}$. Then 
\begin {itemize} 
\item The set $\FS_X:=\{f \in \FS^G_1\mid f(x_\alpha)=x_\alpha, \alpha \in I\}$ is convex subgroup of $\FS^G_1$.
\item If $X_1\subsetneq X_2$, then $\FS_{X_2} \subsetneq \FS_{X_1}$.
\end {itemize}
\end {corollary}
\begin{remark}
    The results above are useful for the one-dimensional transform CDT and for the multi-dimensional R-CDT transforms. {In particular, when signal classes conform to algebraic generative models using any of the convex groups $\cH$ above, Theorem \ref{ConvexGroup1dSet} will guarantee that the signal classes are convex in the transform domain. This convexity property and the one-to-one correspondence between signals and their transforms can facilitate data classification and signal estimation problems. In particular, data classes that are disjoint in signal domain remain disjoint in transform domain.}
\end{remark}
\subsection{Proofs of Section \ref {Sec: oneD}}
\subsubsection{Proof of Theorem \ref{ConvexGroup1dSet}}
We start by proving the following lemma:

\begin{lemma}\label{lm: 1d_composition_prop}
 	Let $p \in \mP_1$,  then 
	\begin{equation}\label{eq: transport_commute_1}
		\widehat p_h = h^{-1}\circ \widehat p, \quad \forall h\in \FS^G_1
	\end{equation}
	where $\widehat p$ denotes the CDT with respect to the reference $r$.
\end{lemma}
\begin{proof}
Let $r\in \mP_1$, and $\widehat p, \widehat p_h \in \FS_1$ be the optimal transports from $r$ to $p, p_h$ respectively.  We have that by definition $\mu_p=\widehat p_\#\mu_r$, $\mu_p = h_{\#}\mu_{p_h}$ and that  $\mu_{p_h}=(\widehat p_h)_\#\mu_{r}.$ Using the well-known relation $(S\circ T)_\#\mu_q=S_\#(T_\#\mu_q)$ for any maps $S,T$ and the fact that  $h\circ \widehat p_h \in \FS_1$ (since both $h$ and $\widehat p$ are non-decreasing) and the property that $h\circ \widehat p_h$ is square-integrable with respect to $\mu_p$, we conclude that $h\circ \widehat p_h=\widehat p$. It follows that $ \widehat p_h=h^{-1}\circ \widehat p$ for all $h\in \FS^G_1$.
\end{proof}

\begin{proof}[Proof of Theorem \ref{ConvexGroup1dSet}]
Assume that $\MS^{-1}$ is convex. Then for $p_{h_1}, p_{h_2} \in \GM_{p,S}$, and $0\le \alpha\le 1$, we have 
\[
\begin{split}
 \alpha \widehat p_{h_1}(x)+(1-\alpha)\widehat p_{h_2}(x)=&\alpha h_1^{-1}\circ \widehat p(x)+(1-\alpha) h_2^{-1}\circ \widehat p(x)\\
 =&\alpha h_1^{-1}\big(\widehat p(x)\big)+(1-\alpha) h_2^{-1}\big(\widehat p(x)\big) \\
 =&\big(\alpha h_1^{-1} +(1-\alpha) h_2^{-1}\big)\circ \widehat p (x).
\end{split}
\]
Thus, $\widehat \GM_{p,\MS}$ is convex. 

For the converse statement, assume that $\widehat \GM_{p,\MS}$ is convex, i.e.,   $\big(\alpha h_1^{-1} +(1-\alpha) h_2^{-1}\big)\circ \widehat p\in \widehat \GM_{p,\MS}$ for all $h_1, h_2\in \MS$.  Since $\widehat \GM_{p,\MS}$ is convex for every $p\in \mP_1$, by varying $p$ (e.g., by choosing p as translations of $r$) one can conclude that $\alpha h_1^{-1} +(1-\alpha) h_2^{-1}\in \MS$ for all $h_1, h_2\in \MS$. 

\end{proof}
\subsubsection{Proof of Proposition \ref {InfSubgrp}}
\begin{proof}[Proof of Proposition \ref {InfSubgrp}]
  Let $x_0, x_1\in \R$ with $x_0\ne x_1$. Choose $f,g \in \FS^G_1$ such that $f(x_0)=x_0, f(x_1)\ne x_0$, and  $g(x_0)\ne x_0, g(x_1)=x_1$. Using the notation of (5) in Example \eqref{exsubgrp}, we get $f\in \FS_{x_0}$ but $f\ne \FS_{x_1}$ and similarly, $g\in \FS_{x_1}$ but $g\ne \FS_{x_0}$. Extending the previous argument over an uncountable set $X=\{x_i\mid i \in I\}\subset \R$  provides an uncountable set $\{\FS_{x_i}\}$ of distinct subgroups of $\FS^G_1$. 
 \end {proof}
 
\section{lot and generative models in multi-dimensions }\label{Sec: multi-D}
 As in the one-dimensional case, a  group $\cH\subseteq \FS_d^G$ partitions the set $\mP_d$ into equivalent classes (as in Proposition \eqref {Part}) that are useful for classification problems. However, unlike the one-dimensional case, Equation \eqref {eq: transport_commute_1} does not hold in general. Thus, in order to obtain convexity of $\widehat{\GM}_{p,\MS}$ from the convexity of the set ${\MS}^{-1} $ as in Theorem \ref{ConvexGroup1dSet}, we need to find conditions on $\MS$ such that for the given $p
 \in \mP_d$  the equation 
 \begin {equation} \label{eq:d-comm}
 	\widehat p_h = h^{-1}\circ \widehat p
 \end {equation}
 holds for all $h\in \MS$ and hence to generate convex subsets of $\widehat \mP_d$ when ${\MS}^{-1}$ is convex. In particular we have
 {\begin{theorem}\label{Thm: setmD}
	Let $p\in \mP_d$ and $\mathcal W_p \subset \FS_d^G$ be a set. If  \eqref{eq:d-comm} holds for all $h\in \mathcal W_p $, and  $\mathcal W_p^{-1}$ is a convex set, then $\widehat S_{p,\mathcal W_p}$ is also convex.
\end{theorem}}

As a corollary, when $\mathcal W_p$ is a group, we have
 \begin{theorem}\label{Thm: ConvGroupmDp}
	Let $p\in \mP_d$ and $\cH_p \subset \FS_d^G$ be a group. If  \eqref{eq:d-comm} holds for all $h\in \cH_p $, and  $\cH_p$ is a convex set, then $\widehat S_{p,\cH_p}$ is also convex.
\end{theorem}
{If Equation \eqref{eq:d-comm} is to hold on a set $\MS$ for all $p\in \mP_d$, we get an analog of Theorem \ref{ConvexGroup1dSet}:}
{\begin{theorem} \label{ConvexGroupmdSet}
  Let $\MS\subset \FS^G_d$ be such that Equation \eqref{eq:d-comm} holds on $\MS$ for all $p\in \mP_d$. Then $\widehat{\GM}_{p,\MS}$ is convex for every $p\in \mP_d$ if and only if $\MS^{-1}:=\{s^{-1}\mid \; s \in \MS \}$ is convex. 
\end{theorem}}
 Note that, unlike $\FS_1^G$, $\FS_d^G$ is not a group for $d\ge 2$. Our next goal is to find conditions on $\cH$ such that $\widehat S_{p,\cH}$ is convex for all $p \in \mP_d$. On the other hand, there are groups that are subsets of $\FS_1^G$. For examples,  the diffeomorphisms  group of translations and isotropic scaling{s} $\Ha:=\{L_{a,u}(x):=ax+u\mid  a>0, u \in \R^d\}$. To see this, let $h\in \cH_a$, and let $h^{-1}(x):= \alpha x + u$ for some $\alpha >0$ and $u\in \R^d$. For $p\in \mP_d$, we have that $\widehat p = \grad \phi_p$ for some convex function $\phi_p$. Hence $$h^{-1}\circ \widehat p = h^{-1}\circ \grad \phi_p = \alpha \grad \phi_p +u.$$ 
To see that $h^{-1}\circ \widehat p=\widehat p_h$,  we use the second part of Brenier's Theorem above \ref {thmBrenier} and simply note that $h^{-1}\circ \widehat p=\alpha \grad \phi_p +u$ is the gradient of the convex function $\psi(x)= \alpha \phi_p(x)+ u\cdot x$. We have the following Theorem.
\begin{theorem}\label{thm: nD_group}
	Let $d\geq 2$ and $\MS \subseteq \FS_d^G$. If for any $p\in \mP_d$ Equation \eqref{eq:d-comm}
	 holds for all $h\in \MS$, then $\MS \subseteq \Ha$. 
\end{theorem}
\begin{remark}
    Note that in the previous theorem, $\MS\subset \FS^G_d$ does not need to be a group. {Theorem \ref{ConvexGroupmdSet} gives a necessary and sufficient condition on $\MS$  under which  $\widehat S_{p,\MS}$ is convex for all $p\in \mP_d$, when Equation \eqref{eq:d-comm} holds.} {Condition \eqref{eq:d-comm} has been referred as the composition property for CDT \cite{park2017}, which holds naturally for all $p$ and $\cH\subset \mP_1$ in the one-dimensional case.  Recently, but after our manuscript had appeared on arXiv, Moosm\"{u}ller et al. have also uploaded \cite{moosmuller2020linear} to arXiv in which they derived similar results concurrently as in Theorem \ref{Thm: setmD}. In addition, one of  the open problems mentioned  by the authors of  \cite{moosmuller2020linear} is resolved in this paper. Specifically, we show in Theorem \ref{thm: nD_group} that the largest set such that the composition property holds for all $p$ when dimension $d\geq 2$ is $\Ha$, i.e., the set of translations and isotropic scalings.}
\end{remark}
\begin{corollary}
    Let $d\geq 2$ and $\cH\subseteq \FS_d^G$ be a subgroup. If for any $p\in \mP_d$ Equation \eqref{eq:d-comm}
	 holds for all $h\in \cH$, then $\cH \subseteq \Ha$.
\end{corollary}
\begin{remark}\label{rmk: relaxation}
    Theorem \ref{thm: nD_group} and its corollary show that the partitioning of the set of transformed signals $\widehat \mP_d$ into convex sets is much more constrained than in the one-dimensional case. However, by allowing condition \eqref{eq:d-comm} to hold only on a subset of $\mP_d$, one can enlarge the set of convex generative models.
\end{remark}

 \subsection{Relaxation in dimension  \texorpdfstring{$d=2$}{d2}}\label{sect: relax}
As mentioned in Remark \ref{rmk: relaxation}, one can relax the condition in \eqref{eq:d-comm} to hold on a subset of $\mP_d$ rather than all of $\mP_d$. This  group is strictly larger than $\cH_a$ satisfying Equality \eqref{eq:d-comm}. In this section, we show how to construct such subset $\mP_r\subset\mP_2$ and group $\cH_r$ strictly larger than $\cH_a$ such that Equation \eqref{eq:d-comm} holds for all $h\in \cH_r$ and $p\in \mP_r$.

\begin{definition}[Restrictive sets of transformations and PDFs]\label{def:Hr}
\begin{equation}
    \cH_r= \left\{h(x,y):=\frac{1}{2}\begin{bmatrix}
			f^{\prime}(x+y)+g^{\prime}(x-y)\\
			f^{\prime}(x+y)-g^{\prime}(x-y)
		\end{bmatrix}\mid f,g \in \mathcal R\right\}
\end{equation}
where $\mathcal R = \{f\in C^2(\R)\mid  f^{\prime} ~\textrm{is a strictly increasing bijection on}~ \R \}$.
\begin{equation}
		\mP_r:= \{p\in \mP_2 : |\det J_{h}|(p\circ h) = r \text{ for some } h\in \cH_r  \}.
	\end{equation}
\end{definition}
\begin{remark}\label{rmk: propchr}
$\cH_r$ has the following properties:
\begin{itemize}
    \item[i)] for any $h\in \cH_r$, $h = \grad \phi$ for the convex function $\phi(x,y)=f(x+y)+g(x-y)$. The fact that $\phi$ is a convex function on $\R^2$ follows from the fact that $f^{\prime}, g^{\prime}$ are strictly increasing. 

    \item[ii)] for any $h_1,h_2\in \cH_r$, $h_1\circ h_2 = \grad \psi$ for some convex function $\psi:\R^2\rightarrow \R$. To see this, one can check that \begin{equation} \label {ExampleinR2}
    (h_1\circ h_2)(x,y)= \frac{1}{2}\begin{bmatrix}
			f_1^{\prime}(f_2^{\prime}(x+y))+g_1^{\prime}(g_2^{\prime}(x-y))\\
			\fp{1}(\fp{2}(x+y))-\gp{1}(\gp{2}(x-y))
		\end{bmatrix},
\end{equation}
where $h_i(x,y)= \frac{1}{2}\grad\big(f_i(x+y)+g_i(x-y)\big)$ and $f_i,g_i\in \mathcal R$ for $i = 1,2$. Since $f_i^{\prime},g_i^{\prime}$ are strictly increasing, so are $f_1^{\prime}\circ f_2^{\prime}$ and $g_1^{\prime}\circ g_2^{\prime}$. The conclusion that $h_1\circ h_2 = \grad \psi$ for some convex function $\psi$ follows from Part i) and identity \eqref {ExampleinR2} above.

\item[iii)] for any $h\in \cH_r$, $h^{-1}\in \cH_r$. To see this, a direct computation gives
\begin{equation}\label{eq: fpinverse}
	 h^{-1}(z,w)=\frac{1}{2}
	\begin{bmatrix}
		\fpi(z+w)+\gpi(z-w)\\\fpi(z+w)-\gpi(z-w)
	\end{bmatrix}.
\end{equation} Since $f^{\prime}, g^{\prime}$ are strictly increasing bijections on $\R$, so are $\fpi,\gpi$, and  the conclusion then follows from Part i).
    
\end{itemize}
\end{remark}
By part i) of the previous remark, it follows that every $h\in \cH_r$ is a conservative vector field (i.e., $h = \grad \phi$ for some $C^1$ function $\phi$) and hence is irrotational since $\grad \times h = \grad \times \grad \phi = 0$.  Figure \ref{fig:vecfield_ex} shows an example of such vector fields generated by some $f,g$ which are invertible and whose derivatives are strictly increasing on $[-5,5]$.

\begin{figure}[!hbt]
    \centering
    \includegraphics[width=10cm]{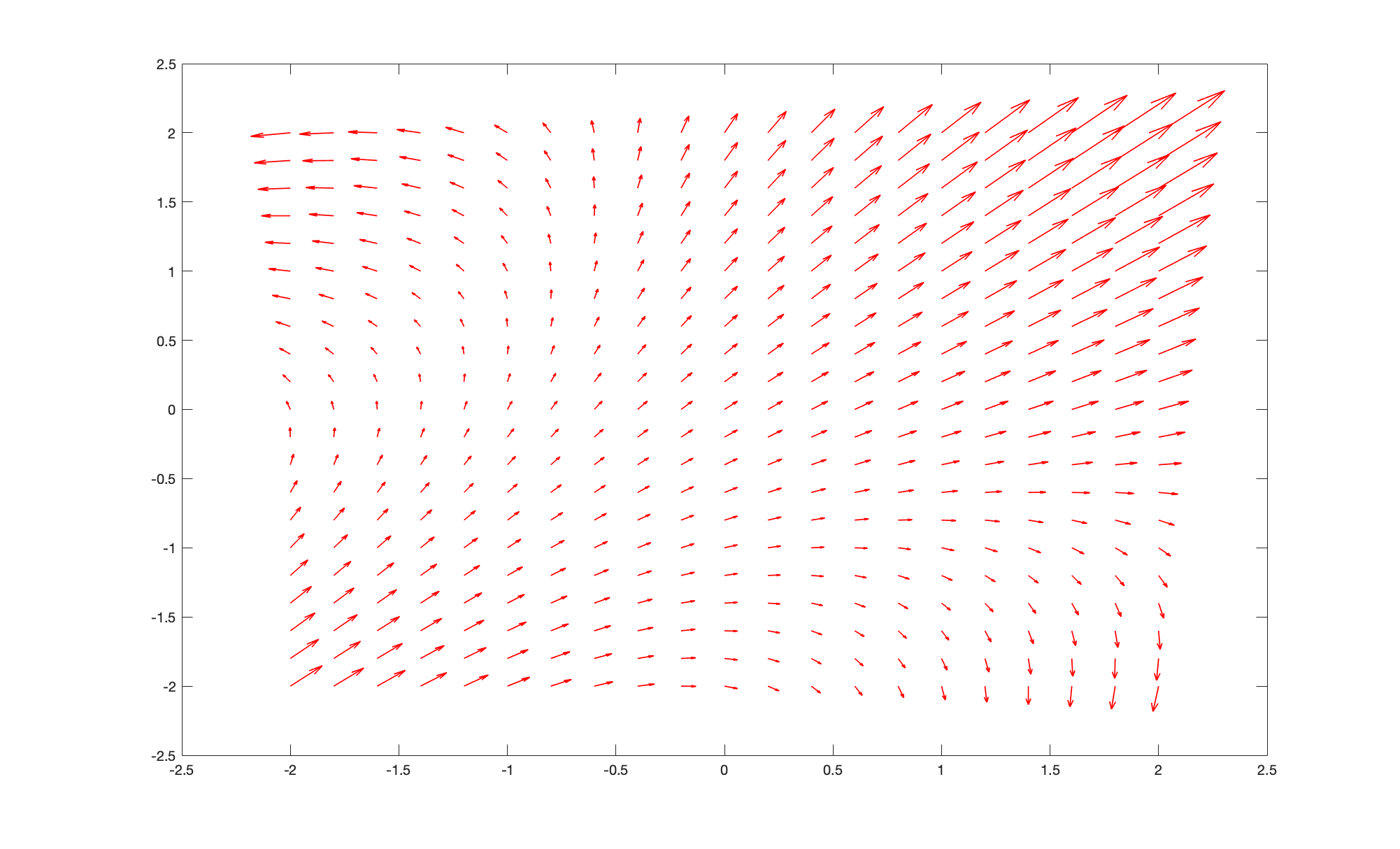}
    \caption{A vector field $h$ on a grid $[-2,2]\times [-2,2]$ generated with $f^{\prime}(t)= t+ 0.1t^2$ and $g^{\prime}(t) = t$.}
    \label{fig:vecfield_ex}
\end{figure}
 By part iii) of  Remark \ref {rmk: propchr}, we have that $\cH_r$ is a group. From the definition of $\mP_r$ and  Brenier's Theorem \ref{thmBrenier}, we have that  $\widehat p\in \cH_r$. The fact that $h\circ \widehat p$ is the gradient of some convex function follows from part ii) of Remark \ref{rmk: propchr}. Thus,  for any $p\in \mP_r$, Equation \eqref{eq:d-comm} holds for all $h\in \cH_r$. In summary, we have the following theorem:
\begin{theorem}\label{thmn2hatformula}
	For any $p\in \mP_r$, we have  $\widehat p_h = h^{-1}\circ \widehat p$ for
any $h\in \cH_r$ where $p_h=|\det J_h|\cdot p\circ h$.	
\end{theorem}

By an  argument similar to the one for Theorem \ref{ConvexGroup1d}, we have the following corollary of Theorem \ref {thmn2hatformula}:
\begin{corollary}
    Let $\cH\subseteq \cH_r$ be a group. Then $\widehat \GM_{p,\cH}$ is convex for any $p\in \mP_r$ if and only if $\cH$ is convex.
\end{corollary}

\subsubsection{Convex Subgroups of $\cH_r$}
We remark first that the convex group $\cH_a$ of translation and isotropic scaling diffeomorphisms in dimension $2$ is a subgroup of $\cH_r$. In particular,  by choosing $f_{a,b_1}(t)= \frac 1 2 at^2+b_1 t$ and $g_{a,b_2}(t)= \frac 1 2 at^2+b_2t$ where $a>0$ and $b_1,b_2\in \R$, one obtains that 
$h(x,y) =\frac{1}{2}\grad\big(f_{a,b_1}(x+y)+g_{a,b_2}(x-y)\big)= a \begin{bmatrix}
x\\y
\end{bmatrix}+\frac 1 2 \begin{bmatrix}b_1+b_2\\b_1-b_2
\end{bmatrix}.$

Indeed, we can construct other examples of convex subgroups of $\cH_r$ by judiciously choosing $f,g$.
\begin{example} \label {example2D}
	Let $\mathcal R_s = \{f(t)= at^2+bt\mid a>0, b\in \R \}$ and $\cH_s = \{h(x,y)= \frac{1}{2}\grad\big(f(x+y)+g(x-y)\big)\mid f,g \in \mathcal R_s\}$
\end{example}

By direct computation, it is easy to see that every $h\in \cH_s$ is of the form 
	\begin{equation}
		h(x,y) = (a_1+a_2)\begin{bmatrix}
			x\\y
		\end{bmatrix}+(a_1-a_2)\begin{bmatrix}
			y\\x
		\end{bmatrix}+ \begin{bmatrix}
			b_1-b_2\\b_1+b_2
		\end{bmatrix},
	\end{equation}
	and vice versa, where $a_1,a_2>0$ and $b_1,b_2\in \R$. Equivalently, $h(x,y)=A\begin{bmatrix}
	x\\y
	\end{bmatrix}+ u$, where $A =\begin{bmatrix}
	a_1+a_2 & a_1-a_2\\a_1-a_2 & a_1+a_2
	\end{bmatrix} $ and $u = \begin{bmatrix}
	b_1-b_2\\b_1+b_2
	\end{bmatrix}.$ 
	It is not difficult to show  that $\cH_s$ is a convex group of diffeomorphisms under the composition operation.

\subsection{Proofs of Section \ref {Sec: multi-D}}
\subsubsection {Proof of Theorems \ref {thm: nD_group}}
The proof of Theorem \ref {thm: nD_group} relies of the following proposition which will be proved toward the end of this section. Through the rest of the section,  $\phi_p$ will denote a convex function such that $\grad\phi_p$ is the optimal transport map between a fixed reference $r$ and a function  $p\in \mP_d^{*}$ ($d\ge2)$, where
\begin{equation}
    \mP_d^{*}:= \{p \in \mP_d\mid r= |\det \grad f|(p\circ f) ~\textrm{for some} ~f \in \FS_d^G \}.
\end{equation}

\begin{proposition}\label{prop: composition_gradient}
	Let $\varphi: \R^d \rightarrow \R$ be a convex function in $C^2(\R^d)$ such that for any $p\in \mP_d^{*}$, $\grad \varphi \circ \grad \phi_p$ can be written as $\grad \varphi \circ \grad \phi_p=\grad\gamma$ for some function $\gamma=\gamma(p)$. Then $\grad^2 \varphi(x)\equiv \alpha I_d$, where $I_d$ is the identity matrix in $\R^{d\times d}$. \end{proposition}
We are now ready to prove Theorem \ref {thm: nD_group}.
\begin{proof}[Proof of Theorem  \ref{thm: nD_group}]
Recall that, by the assumption on  $\MS \subset \FS^G_d$, for any  $p\in \mP_d$ the following holds 
$$\widehat p_h = h^{-1}\circ \widehat p,\quad \text { for all } h\in\MS.$$
Using Brenier's Theorem \ref{thmBrenier}  above,  a transport map is optimal in the sense of \eqref{Monge} if and only if it is the gradient of a convex function. Accordingly, there exist convex functions $\phi_{p_h}, \varphi, \phi_p$ such that $\grad \phi_{p_h}= \widehat p_h$, $\grad \varphi = h^{-1}$ and $\grad \phi_p = \widehat p$ (note that like $h$, $h^{-1}$ is also an optimal transport map). In particular, we have that $\grad \varphi \circ \grad \phi_p=\grad \phi_{p_h}$ for every $p\in \mP_d^{*}$. By \cref{prop: composition_gradient}, it follows that $\grad^2 \phi \equiv \alpha I_d$ ($\alpha >0$). Hence $\grad \varphi (x)= \alpha x + b$, where $b\in \R^d$, i.e., $h^{-1}\in \cH_a$, which also implies $h \in \cH_a$. Thus $\MS \subseteq \cH_a$.
\end{proof} 
\subsubsection{ Proof of Proposition \ref{prop: composition_gradient}}
We start by proving the following two lemmas.
\begin{lemma}\label{lm:hessiancommute}
	Let $\varphi, \phi: \R^d\rightarrow \R$ be two  functions in $C^2(\R^d)$. If $\grad \varphi \circ \grad \phi = \grad \gamma$ for some  function $\gamma \in C^2(\R^d)$, then the matrix-valued functions  $(\grad^2\varphi)\circ \grad\phi$ and $\grad^2 \phi$ must commute, i.e., $\grad^2\varphi(\grad \phi(x)) \grad^2\phi(x) = \grad^2\phi(x)\grad^2\varphi(\grad\phi(x))$ for all $x\in \R^d$.
\end{lemma}
\begin{proof}
	Since $\varphi,\phi,\gamma$ are continuous twice differentiable, by Schwarz's theorem, their Hessian matrices $\grad^2 \varphi$, $\grad^2\phi, \grad^2 \gamma$ are all symmetric. It follows from $\big(\grad \varphi\big) \circ \grad \phi = \grad \gamma$ that $\grad^2\varphi(\grad \phi(x)) \grad^2\phi(x) = \grad^2 \gamma(x)$ for all $x
	\in \R^d$  by multivariate chain rule. Since the product of two real symmetric matrices is symmetric if and only if they commute,   $\grad^2\varphi(\grad \phi(x)) $ and $\grad^2 \phi(x)$ must commute.
\end{proof}

 \begin{lemma}\label{lm: diagm}
	Let $A$ be an $\R^{d\times d}$ matrix such that $A\Delta=\Delta A$ for some diagonal matrix  $\Delta$ with distinct diagonal entries.  Then $A$ is a diagonal matrix.
\end{lemma} 

\begin{proof}
	Let  $\Delta= \begin{bmatrix}
		\delta_1\\ \ & \delta_2\\ &&\ddots\\&&&\delta_d
	\end{bmatrix}$ where  $\delta_i\neq \delta_j$ whenever $i\neq j$.  Since $A\Delta = \Delta A$, by comparing of the $(i,j)$-th entry of both sides we have 
	\begin{equation}
		\delta_ja_{ij}=\delta_ia_{ij},
	\end{equation}
where $a_{ij}$ denotes the $(i,j)$-th entry of matrix $A$. Since 
$\delta_i\ne \delta_j$,  for $i\neq j$, it follows that for $i\neq j$, $a_{ij}=0$.
\end{proof}

\begin{lemma}\label{lm: mI}
	Let $D$ be a diagonal matrix in $\R^{d\times d}$ such that  $DM=MD$ for some matrix $M$ having the property that it has an eigen-space $E_\lambda=span \{u\}$ and such that all entries of $u$ are non-zero.  Then $D=\alpha I$ where $\alpha \in \R$ and $I_d$ is the identity matrix in  $\R^{d\times d}$.
\end{lemma}
\begin{proof}
 Since $DM = MD$, it follows that 
\begin{equation}
	MDu=DMu = D\lambda u = \lambda D u.
\end{equation}
Hence $Du\in E_{\lambda}(M)$ and $Du = \alpha u$ for some $\alpha \in \R$ since $\dim E_{\lambda}(M) = 1$. Using the fact all entries of $u$ are non-zero and comparing the entries of $Du$ and $\alpha u$, one immediately gets that $D = \alpha I_d$.
\end{proof}
As stated at the beginning of this section,  $\phi_p$ will denote a convex function such that $\grad\phi_p$ is the optimal transport map between $r$ and $p$.

\begin{proof}[Proof of Proposition \ref {prop: composition_gradient} ]
	By Lemma \ref{lm:hessiancommute}, we have that the matrices $\grad^2\varphi\big( \grad\phi_p (x)\big)$ and $\grad^2 \phi_p(x)$ must commute for every $x\in \R^d$, and for every $p\in \mP_d^{*}$. To prove the proposition, we make judicious choices for $p$. First, we choose $p=p(\delta_1,\cdots,\delta_d)$ with $\phi_p (x) = \frac{1}{2} \sum\limits_{i=1}^n \delta_ix_i^2$ where $\delta_j>0$ for $j=1,\cdots,d$. In particular, $r(x) = |\det \grad^2\phi_p(x)| p(\grad \phi_p(x))$ where $r\in \mP_d$ is the reference. Since $\phi_p$ is convex and quadratic, it is not difficult to show that $p \in \mP^{*}_d$. It is easy to see that $\grad \phi_p (x)= \begin{bmatrix}
		\delta_1x_1\\ \vdots \\ \delta_dx_d
	\end{bmatrix}$ and $\grad^2\phi_p(x)=\begin{bmatrix}
		\delta_1\\ \ & \delta_2\\ &&\ddots\\&&& \delta_d
	\end{bmatrix}$. By Lemma \ref{lm: diagm}, setting $A=\grad^2\varphi(\grad \phi_p(x))$ and  $\Delta=\phi_p(x)$, we conclude that $A=\grad^2\varphi(\grad \phi_p(x))$ is a diagonal matrix for every $x\in \R^d$.  Since $\grad\phi_p(x) :\R^d\rightarrow \R^d$ is bijective, it follows that $\grad^2\varphi(x)$ is also diagonal for every $x\in \R^d$. Next, we choose $p=p(M)$ such that $\phi_p(x)=\frac 1 2 x^tMx$, where $M$ is a constant positive definite matrix with an eigenvector $u$ whose entries are non-zero and  whose corresponding eigenspace has dimension $1$. A simple construction using the spectral decomposition of symmeric matrices show that such an $M$ exists. For these choices of $p$, $\phi_p$, we have  $\grad^2\phi_p(x)=M$ which is a constant matrix independent of $x$.  Again, it is not difficult to show that $p\in \mP^{*}_d$.   Hence  $\grad^2\varphi(Mx)$ and $M$ commute. Since $\grad^2\varphi(Mx)$ is diagonal for every $x\in \R^d$,  using Lemma  \ref{lm: mI} with $D=\grad^2 \varphi(x)$ , we have that $\grad^2\varphi(Mx) = \beta_x I$, where $\beta_x$ is a constant depending on $x$. Since $M$ is an invertible matrix, it follow that $\grad^2\varphi (x)=\alpha_x I_d$ where $\alpha_x$ is a constant depending on $x$. Since $\frac{\partial^2 \varphi}{\partial x_i\partial x_j}\equiv 0$ for $i\neq j$, we have that $\varphi(x_1,...,x_d) = F_1(x_1)+\cdots+F_n(x_d)$ for some univariate functions $F_1,...,F_d$.
	Hence $F_1^{\dprime}(x_1)=F_2^{\dprime}(x_2)=\cdots = F_d^{\dprime}(x_d)$ for all $(x_1,x_2,...,x_d)\in \R^d$. If follows that $F_i^{\dprime}$ must be the same constant function for
	$i=1,...,d$ and that $\alpha_x=\alpha$ is independent of $x$, which implies that $\grad^2 \varphi(x) \equiv \alpha I_d$ for some constant $\alpha$. 	

	\end{proof}

\section{Summary and Open questions}\label{Sec: conclusion}

In this paper we have worked to highlight and clarify certain properties of an emerging set of transport-based signal transforms. More specifically, we have worked to show that for certain types of signals and generative models, the transport transforms discussed earlier have the ability to render signal classes convex in transform space. As convex signal classes render solutions to estimation and detection problems much simpler to solve (e.g. via linear least squares, or linear classification), the topic is important. {For illustration purposes, let us consider an image class consisting of random translations and scalings of a particular digit in the MNIST dataset.   When the training set is limited to include only relatively small variations (in this case,  small translations and  scalings) of this digit, positive linear combinations of the training samples in transform domain can be used to model the data class beyond the available training samples. In this way, the framework can be expanded  to include data corresponding to both small variations present  in the training  set as well as larger variations not present in the training set but may be present in a testing set.
In that sense the convexity property of the transform allows us to extrapolate the model beyond the available training data. In general, depending on the specific model assumptions (e.g., whether isotropic scaling diffeomorphisms are present in the generative model $\GM_{p,\cH}$), taking convex or positive linear combinations of the limited training data produces possibly many unobserved data of the corresponding data class (in transform domain). In summary, if the two disjoint data classes conform to the convexity condition as well as the algebraic generative model, data generated by the training data in the two classes according the above process stay in two disjoint convex sets in transform domain.}

While the picture is far from complete, we hope that this paper serves a starting point on guiding meaningful applications of these transforms and opens the door to deeper understanding about when and how to apply these transforms.

\subsection{List of contributions}
More specifically, we make the connection between the convex group structure in our generative model and convex partitions of the transformed signal space via the formula \eqref{eq: transport_commute_1}, which holds naturally in the one-dimensional case (see Lemma \ref{lm: 1d_composition_prop}). Furthermore, we give examples of convex groups of diffeomorphisms in the one-dimensional case and show that there are infinitely many such groups.

In dimension $d\geq 2$, we show that the only groups $\cH$ of diffeomorphisms that validate formula \eqref{eq: transport_commute_1} for all $p\in \mP_d$ and all  $h\in \cH$ are subgroups of $\cH_a$, which is {the} group of translations and isotropic scaling diffeomorphisms. In particular, any convex subgroup of $\cH_a$ generates a convex partition of the transformed signal space $\widehat\mP_d$. 

Moreover, in dimension two, we show how to construct a  group $\cH_r\subset \FS_2$ which is larger than $\cH_a$ such that \eqref{eq: transport_commute_1} holds for all $h\in \cH_r$ and  $p$ in $\mP_r$ which is a subset of $\mP_d$. In particular, any convex subgroup of $\cH_r$ generates a convex partition of $\widehat\mP_r$. {In Example \ref{example2D}, by judiciously picking $f,g$ as in the construction of $\cH_r$ (cf. Definition \ref{def:Hr} ), a convex subgroup $\cH_s$ of {the set of} affine transformations is given. Moreover the set of matrices $A$ defining $h$ in $\cH_s$ forms a commutative convex subgroup of the set of symmetric positive definite matrices. Such matrix A   is either a matrix that corresponds to an isotropic scaling or a matrix which has positive eigenvalues with eigenvectors $\begin{bmatrix}
    1\\1
\end{bmatrix}$ and $\begin{bmatrix}
    1\\-1
\end{bmatrix}.$ In summary, $\cH_s$ includes compositions of isotropic scalings, translations and stretching  in the directions $\begin{bmatrix}
    1\\1
\end{bmatrix}$ and $\begin{bmatrix}
    1\\-1
\end{bmatrix}$.  We leave the construction of more general groups in $\FS_2$ that yield convex partitions of $\mP_r$ and concrete physical applications as future research topics.  }

{Note that as long as the composition property \eqref{eq: transport_commute_1} holds, the convexity of $\cH^{-1}$ implies the convexity of $\widehat \GM_{p,\cH}$. Though a convex group partitions $\widehat \mP_d$ into convex equivalent classes, there are situations where the group structure is not needed. For example, if a signal class $\GM_{p,\cH}$ can be generated by a specific template $p$ under a set of transportations $\cH$, $\cH^{-1}$ being convex guarantees that $\widehat \GM_{p,\cH}$ is convex ($\cH$ does not have to be a group). In contrast, if any signal in a signal class $\GM_{p,\cH}$ can be a generating template, it is not hard to see  that $\cH$ must indeed be a group. This property {that $\cH^{-1}$ being convex  }  often enables straightforward practical solutions to engineering problems obviating the need for computationally expensive, nonlinear, non-convex, optimization methods \cite{Rubaiyat20,Nichols:19}.}

\subsection{Open questions}
In 1D, a characterization of convex subgroups of $\FS_1$ is missing. More specifically, one can ask the following:
\begin{enumerate}
    \item Besides $\FS_1$,  are there subgroups of $\FS_1$ that are not of the form of those in Example \ref{exsubgrp}?
    \item If the answer to the previous question is yes, can we give more examples of  convex subgroups, and can we characterize the subgroups into a few concrete categories? 
    \item For every group $\cH$ the set $\widehat \mP_1$ is tiled  by a convex structure in $\widehat \mP_1$. What is the geometry of this structure?

\end{enumerate}

In multi-dimensions, while it is convenient to make use of formula \eqref{eq: transport_commute_1} to derive convexity results similar to the ones in one dimension, it is not necessary. In particular, one can ask the following questions:
\begin{enumerate}
    \item Can one derive that the convexity of $\widehat \GM_{p,\cH}$ from the convexity of $\cH$ without formula \eqref{eq: transport_commute_1} being true for all $h\in \cH$?
    
    \item Are there other conditions one can impose on the model $\GM_{p,\cH}$ other than that $\cH$ is convex so that $\widehat \GM_{p,\cH}$ is convex?

\end{enumerate}

There are several natural questions related to the relaxation in dimension two:
\begin{enumerate}
\item Are there more interesting examples of subgroups of $\cH_r$ that have connections to possible applications?

\item For dimension $d>2$, using Example \ref{example2D} as a guide, what are the set of convex subgroups  $\cH$ of $\FS^G_d$  when we   restrict the signals to some subsets of $\mP_d$?
\end{enumerate}

The above discussions fall under the framework of an algebraic generative model $\GM_{p,\cH}$ with a single template $p$ and  a convex group $\cH$ of diffeomorphisms. In fact, it might be suitable,  to consider multiple templates for the generative modeling in certain applications,  or to not assume a group structure for $\cH$ when modeling certain image classes.  We leave such extensions for future research topics and believe that delving into these questions could potentially lead to more thoughtful engineering modeling, algorithmic design and new interesting mathematics.

\section*{Acknowledgment}
This work is supported by NIH award R01 GM130825. The authors would also like to thank Longxiu Huang, Soheil Kolouri, Armenak Petrosyan, and Mohammad Shifat-E-Rabbi for their careful reading of our paper and their pertinent suggestions.


\bibliographystyle{abbrv}
\bibliography{references}

\begin{thebibliography}{10}

\bibitem{arjovsky2017wasserstein}
M.~Arjovsky, S.~Chintala, and L.~Bottou.
\newblock Wasserstein {GAN}.
\newblock {\em arXiv preprint arXiv:1701.07875}, 2017.

\bibitem{basu2014}
S.~Basu, S.~Kolouri, and G.~Rohde.
\newblock Detecting and visualizing cell phenotype differences from microscopy
  images using transport-based morphometry.
\newblock {\em Proc. Natl. Acad. Sci. U.S.A.}, 111(9):3448--3453, 2014.

\bibitem{Belhumeur97}
P.~N. {Belhumeur}, J.~P. {Hespanha}, and D.~J. {Kriegman}.
\newblock Eigenfaces vs. fisherfaces: recognition using class specific linear
  projection.
\newblock {\em IEEE Transactions on Pattern Analysis and Machine Intelligence},
  19(7):711--720, 1997.

\bibitem{brenier1991}
Y.~Brenier.
\newblock Polar factorization and monotone rearrangement of vector-valued
  functions.
\newblock {\em Commun. Pure Appl. Math.}, 44(4):375--417, 1991.

\bibitem{cortes1995}
C.~Cortes and V.~Vapnik.
\newblock Support-vector networks.
\newblock {\em Mach. Learn.}, 20(3):273--297, 1995.

\bibitem{emerson2020turbulence}
T.~H. Emerson and J.~M. Nichols.
\newblock Fitting local, low-dimensional parameterizations of optical
  turbulence modeled from optimal transport velocity vectors.
\newblock {\em Pattern Recognition Letters}, 133:123--128, 2020.

\bibitem{fisher1936}
R.~A. Fisher.
\newblock The use of multiple measurements in taxonomic problems.
\newblock {\em Annals of eugenics}, 7(2):179--188, 1936.

\bibitem{Guan2019}
S.~{Guan}, B.~{Liao}, Y.~{Du}, and X.~{Yin}.
\newblock Vehicle type recognition based on {R}adon-{CDT} hybrid transfer
  learning.
\newblock In {\em 2019 IEEE 10th International Conference on Software
  Engineering and Service Science (ICSESS)}, pages 1--4, 2019.

\bibitem{haker2004}
S.~Haker, L.~Zhu, A.~Tannenbaum, and S.~Angenent.
\newblock Optimal mass transport for registration and warping.
\newblock {\em Int. J. Comput. Vis.}, 60(4):225--240, 2004.

\bibitem{hastie2009elements}
T.~Hastie, R.~Tibshirani, and J.~Friedman.
\newblock {\em The elements of statistical learning: data mining, inference,
  and prediction}.
\newblock Springer Science \& Business Media, 2009.

\bibitem{kantorovich1942translation}
L.~V. {K}antorovich.
\newblock On translation of mass (in {R}ussian), {C R. Doklady}.
\newblock {\em Acad. Sci. USSR}, 37:199--201, 1942.

\bibitem{kolouri2016c}
S.~Kolouri, S.~Park, and G.~Rohde.
\newblock The {R}adon cumulative distribution transform and its application to
  image classification.
\newblock {\em IEEE Trans. Image Process.}, 25(2):920--934, 2016.

\bibitem{kolouri2017optimal}
S.~Kolouri, S.~R. Park, M.~Thorpe, D.~Slepcev, and G.~K. Rohde.
\newblock Optimal mass transport: Signal processing and machine-learning
  applications.
\newblock {\em IEEE Signal Processing Magazine}, 34(4):43--59, 2017.

\bibitem{kolouri2016b}
S.~Kolouri, A.~Tosun, J.~Ozolek, and G.~Rohde.
\newblock A continuous linear optimal transport approach for pattern analysis
  in image datasets.
\newblock {\em Pattern Recognit.}, 51:453--462, 2016.

\bibitem{kundu2018discovery}
S.~Kundu, S.~Kolouri, K.~I. Erickson, A.~F. Kramer, E.~McAuley, and G.~K.
  Rohde.
\newblock Discovery and visualization of structural biomarkers from {MRI} using
  transport-based morphometry.
\newblock {\em NeuroImage}, 167:256--275, 2018.

\bibitem{mccullagh1989generalized}
P.~McCullagh and J.~A. Nelder.
\newblock {\em Generalized Linear Models}, volume~37.
\newblock CRC Press, 1989.

\bibitem{monge1781}
G.~Monge.
\newblock {\em M\'{e}moire sur la th\'{e}orie des d\'{e}blais et des remblais}.
\newblock De l'Imprimerie Royale, 1781.

\bibitem{moosmuller2020linear}
C.~Moosmüller and A.~Cloninger.
\newblock Linear optimal transport embedding: Provable fast wasserstein
  distance computation and classification for nonlinear problems, 2020.

\bibitem{ni2009local}
K.~Ni, X.~Bresson, T.~Chan, and S.~Esedoglu.
\newblock Local histogram based segmentation using the wasserstein distance.
\newblock {\em International journal of computer vision}, 84(1):97--111, 2009.

\bibitem{nichols2019time}
J.~M. Nichols, M.~N. Hutchinson, N.~Menkart, G.~A. Cranch, and G.~K. Rohde.
\newblock Time delay estimation via wasserstein distance minimization.
\newblock {\em IEEE Signal Processing Letters}, 26(6):908--912, 2019.

\bibitem{Nichols:19}
J.~M. Nichols, M.~N. Hutchinson, N.~Menkart, G.~A. Cranch, and G.~K. Rohde.
\newblock Time delay estimation via {W}asserstein distance minimization.
\newblock {\em Signal Processing Letters}, 26(6):908--912, 2019.

\bibitem{ozolek2014}
J.~Ozolek, A.~Tosun, W.~Wang, C.~Chen, S.~Kolouri, S.~Basu, H.~Huang, and
  G.~Rohde.
\newblock Accurate diagnosis of thyroid follicular lesions from nuclear
  morphology using supervised learning.
\newblock {\em Med. Image. Anal.}, 18(5):772--780, 2014.

\bibitem{park2018}
S.~Park, L.~Cattell, J.~Nichols, A.~Watnik, T.~Doster, and G.~Rohde.
\newblock De-multiplexing vortex modes in optical communications using
  transport-based pattern recognition.
\newblock {\em Opt. Express}, 26(4):4004--4022, 2018.

\bibitem{park2017}
S.~Park, S.~Kolouri, S.~Kundu, and G.~Rohde.
\newblock The cumulative distribution transform and linear pattern
  classification.
\newblock {\em Appl. Comput. Harmon. Anal.}, 2017.

\bibitem{park2018multiplexing}
S.~R. Park, L.~Cattell, J.~M. Nichols, A.~Watnik, T.~Doster, and G.~K. Rohde.
\newblock De-multiplexing vortex modes in optical communications using
  transport-based pattern recognition.
\newblock {\em Optics express}, 26(4):4004--4022, 2018.

\bibitem{Rubaiyat20}
A.~H.~M. {Rubaiyat}, K.~M. {Hallam}, J.~M. {Nichols}, M.~N. {Hutchinson},
  S.~{Li}, and G.~K. {Rohde}.
\newblock Parametric signal estimation using the cumulative distribution
  transform.
\newblock {\em IEEE Transactions on Signal Processing}, 68:3312--3324, 2020.

\bibitem{rubner2000}
Y.~Rubner, C.~Tomasi, and L.~Guibas.
\newblock The earth mover's distance as a metric for image retrieval.
\newblock {\em IJCV}, 40(2):99--121, 2000.

\bibitem{santambrogio2015optimal}
F.~Santambrogio.
\newblock {\em Optimal transport for applied mathematicians}.
\newblock Springer, 2015.

\bibitem{schiebinger2019optimal}
G.~Schiebinger, J.~Shu, M.~Tabaka, B.~Cleary, V.~Subramanian, A.~Solomon,
  J.~Gould, S.~Liu, S.~Lin, P.~Berube, et~al.
\newblock Optimal-transport analysis of single-cell gene expression identifies
  developmental trajectories in reprogramming.
\newblock {\em Cell}, 176(4):928--943, 2019.

\bibitem{shen2018wasserstein}
J.~Shen, Y.~Qu, W.~Zhang, and Y.~Yu.
\newblock Wasserstein distance guided representation learning for domain
  adaptation.
\newblock In {\em Thirty-Second AAAI Conference on Artificial Intelligence},
  2018.

\bibitem{shifaterabbi2020radon}
M.~Shifat-E-Rabbi, X.~Yin, A.~H.~M. Rubaiyat, S.~Li, S.~Kolouri, A.~Aldroubi,
  J.~M. Nichols, and G.~K. Rohde.
\newblock Radon cumulative distribution transform subspace modeling for image
  classification.
\newblock 2020.
\newblock Preprint available at \url{https://arxiv.org/abs/2004.03669}.

\bibitem{thorpenotes}
M.~Thorpe.
\newblock Introduction to optimal transport.
\newblock \url{https://www.math.cmu.edu/~mthorpe/OTNotes}, 2018.
\newblock Supplementary notes for Introduction to Optimal Transport Lent 2018
  at the University of Cambridge.

\bibitem{tosun2015detection}
A.~B. Tosun, O.~Yergiyev, S.~Kolouri, J.~F. Silverman, and G.~K. Rohde.
\newblock Detection of malignant mesothelioma using nuclear structure of
  mesothelial cells in effusion cytology specimens.
\newblock {\em Cytometry Part A}, 87(4):326--333, 2015.

\bibitem{villani2003topics}
C.~Villani.
\newblock {\em Topics in optimal transportation}.
\newblock Number~58. American Mathematical Soc., 2003.

\bibitem{wang2013}
W.~Wang, D.~Slep{\v{c}}ev, S.~Basu, J.~Ozolek, and G.~Rohde.
\newblock A linear optimal transportation framework for quantifying and
  visualizing variations in sets of images.
\newblock {\em IJCV}, 101(2):254--269, 2013.

\end{thebibliography}

\end{document}